\documentclass[letterpaper, 10 pt, conference]{ieeeconf}  

\IEEEoverridecommandlockouts                              
\usepackage{graphics} 
\usepackage{epsfig} 
\usepackage{mathptmx} 
\usepackage{times} 
\usepackage{amsmath} 
\usepackage{amssymb}  
\usepackage{mathtools}
\usepackage{xcolor}
\let\labelindent\relax
\usepackage{enumitem}
\usepackage{pgfplots}
\usepackage{tikz}
\usetikzlibrary{shapes,arrows}
\usepackage{caption}
\usepackage{subcaption}
\usepackage{url}

\newtheorem{theorem}{\bf Theorem}
\newtheorem{remark}[theorem]{\bf Remark}

\newtheorem{lemma}[theorem]{\bf Lemma}
\newtheorem{problem}[theorem]{\bf Problem}
\newtheorem{corollary}[theorem]{\bf Corollary}

\newcommand{\bbD}{\mathbb{D}}

\newcommand{\bbN}{\mathbb{N}}

\newcommand{\bbR}{\mathbb{R}}
\newcommand{\bbS}{\mathbb{S}}

\newcommand{\calC}{\mathcal{C}}

\newcommand{\calI}{\mathcal{I}}

\newcommand{\calL}{\mathcal{L}}

\newcommand{\calP}{\mathcal{P}}

\DeclareMathOperator{\CNN}{CNN}

\DeclareMathOperator{\diag}{diag}
\DeclareMathOperator{\chol}{chol}
\DeclareMathOperator{\Cayley}{Cayley}
\newcommand\PP[1]{\textcolor{black}{#1}}

\tikzstyle{block} = [rectangle, draw, fill=blue!20, 
    text width=18em, text centered, rounded corners, minimum height=2em]
\tikzstyle{block2} = [rectangle, draw, fill=red!20, 
    text width=18em, text centered, rounded corners, minimum height=2em]
\tikzstyle{block3} = [rectangle, draw, fill=yellow!20, 
    text width=18em, text centered, rounded corners, minimum height=2em]
\tikzstyle{block4} = [text width=25em, text centered, minimum height=2em]
\tikzstyle{line} = [draw, -latex']

\title{\LARGE \bf
Lipschitz-bounded 1D convolutional neural networks\\
using the Cayley transform and the controllability Gramian
}

\author{Patricia Pauli$^1$, Ruigang Wang$^2$, Ian R. Manchester$^2$ and Frank Allg\"ower$^1$
\thanks{*This work was funded by Deutsche Forschungsgemeinschaft (DFG, German Research Foundation) under Germany's Excellence Strategy - EXC 2075 - 390740016 and under grant 468094890. The authors thank the International Max Planck Research School for Intelligent Systems (IMPRS-IS) for supporting Patricia Pauli.}
\thanks{$^{1}$Patricia Pauli and Frank Allgöwer are with the Institute for Systems Theory and Automatic Control, University of Stuttgart, 70569 Stuttgart, Germany
        {\tt\small patricia.pauli@ist.uni-stuttgart.de}}%
\thanks{$^{2}$ Ruigang Wang and Ian R. Manchester are with the Australian Centre for  Robotics and School of Aerospace, Mechanical and Mechatronic Engineering, The University of Sydney, Australia.}%
}

\begin{document}

\maketitle
\thispagestyle{empty}
\pagestyle{empty}

\begin{abstract}
    We establish a layer-wise parameterization for 1D convolutional neural networks (CNNs) with built-in end-to-end robustness guarantees. In doing so, we use the Lipschitz constant of the input-output mapping characterized by a CNN as a robustness measure. We base our parameterization on the Cayley transform that parameterizes orthogonal matrices and the controllability Gramian of the state space representation of the convolutional layers. The proposed parameterization by design fulfills linear matrix inequalities that are sufficient for Lipschitz continuity of the CNN, which further enables unconstrained training of Lipschitz-bounded 1D CNNs. Finally, we train Lipschitz-bounded 1D CNNs for the classification of heart arrythmia data and show their improved robustness.
\end{abstract}



\section{Introduction}
Robustness of neural networks (NNs) has lately been a topic of increasing importance, for which the Lipschitz constant of the NN's input-output mapping has become a common metric \cite{szegedy2013intriguing}. \PP{For a low  Lipschitz constant, a slightly perturbed input admits only small changes in the output which is desirable in robust NNs.} Finding an accurate upper bound on an NN's Lipschitz constant has been broadly tackled, e.g. using relaxations by quadratic constraints \cite{fazlyab2019efficient,pauli2023lipschitz}, average operators \cite{combettes2020deep} and polynomial optimization \cite{latorre2020lipschitz}. In addition, the training of provably Lipschitz-bounded NNs was proposed by including constraints \cite{pauli2021training,pauli2022neural} and regularization techniques \cite{gouk2021regularisation}. While effective, one drawback of these methods is the computational overhead coming from constraints and projections in the optimization problem \cite{pauli2022neural}.

To overcome this, \cite{revay2020lipschitz,revay2023recurrent,wang2023direct} suggest so-called direct parameterizations for equilibrium networks, recurrent equilibrium networks, and feedforward neural networks, respectively, with guaranteed Lipschitz bounds. From a set of unconstrained variables \cite{revay2020lipschitz,revay2023recurrent,wang2023direct} formulate the NNs in such a way that they by design satisfy linear matrix inequalities (LMIs). These LMIs in turn are sufficient conditions for Lipschitz continuity such that, this way, one can parameterize the class of Lipschitz-bounded NNs with a Lipschitz upper bound predefined by the user. The underlying training problem boils down to an unconstrained optimization problem that can be solved using gradient methods. In this work, we take the same approach as in \cite{revay2020lipschitz,revay2023recurrent,wang2023direct} to parameterize Lipschitz-bounded 1D convolutional neural networks (CNNs).

CNNs play an important role in deep learning and have been tremendously successful in image and audio processing tasks \cite{gu2018recent,kiranyaz20211d,oord2016wavenet}. We distinguish between 1D and 2D CNNs: By 1D CNNs, we mean CNNs with inputs with one propagation dimension such as time signals, whereas 2D CNNs take inputs with two propagation dimensions, e.g., pictures. This paper focuses on 1D CNNs, that  
typically consist of convolutions (= finite impulse response (FIR) filters), nonlinear activation functions, that are slope-restricted, pooling layers, and linear layers that are concatenated in a feedforward structure. 1D CNNs parameterize a rich class of nonlinear systems and are utilized as classifiers for time signals. An extension to 2D CNNs based on \cite{gramlich2022convolutional} is left for future work. 

While numerous methods exist for enforcing Lipschitz continuity and orthogonality in fully connected layers \cite{anil2019sorting}, the design of Lipschitz-bounded convolutional layers and CNNs is less studied and often restricted to special convolutions \cite{trockman2021orthogonalizing}. Recently, this has been approached via parameterization of convolutional layers in the Fourier domain \cite{wang2023direct,trockman2021orthogonalizing}, however, parameterizations in \cite{wang2023direct} are built for a particular input size, whereas our new parameterization in state space is independent of the input dimension, and can even be applied causally to a signal on $[0, \infty)$. 
Another feature of our approach is that we impose Lipschitz continuity directly onto the input-output mapping rather than on all individual layers, like it is done in many other works \cite{araujo2023a}, using that the product of the Lipschitz bounds of the layers yields a Lipschitz bound for the overall NN. In this way, our approach shows reduced conservatism in the compliance with the Lipschitz bound, i.e., our CNNs have higher expressivity for the same Lipschitz bound. In addition our approach accounts for standard pooling layers, 
which were not addressed in other recent Lipschitz-bounded parameterizations of CNNs \cite{wang2023direct}.

Our main contribution is a scalable and expressive layer-wise parameterization of Lipschitz-bounded 1D CNNs that makes use of the Cayley transform to parameterize orthogonal matrices. Beside the Cayley transform, a tool that was used for NN parameterization before, we newly propose to utilize the controllability Gramian in the context of parameterizing convolutional layers of Lipschitz-bounded CNNs. In particular, we reformulate parts of the underlying LMI, that enforces dissipativity onto convolutional layers, as a Lyapunov equation whose unique analytical solution is the controllability Gramian. Using our parameterization, we then train Lipschitz-bounded 1D CNNs solving an unconstrained optimization problem. 

The remainder of the paper is structured as follows: Section~\ref{sec:problem} first introduces 1D CNNs and formally states the training problem. In Section~\ref{sec:prelims}, we discuss preliminaries, including state space represenations for 1D convolutions and Lipschitz constant estimation for 1D CNNs. In Section~\ref{sec:parameterization}, we present our direct parameterization for Lipschitz-bounded 1D CNNs and in Section~\ref{sec:simulation}, we train Lipschitz-bounded 1D CNNs on the MIT-BIH arrhythmia database \cite{mitbihdatabase}, a well-known benchmark dataset for 1D CNNs. Finally, in Section~\ref{sec:conclusion}, we conclude the paper.

\textbf{Notation:} By $\bbD^n$ ($\bbD_+^n$) and $\bbS^n$ ($\bbS^n_+$), we denote the set of $n$-dimensional (positive definite) diagonal and symmetric matrices, respectively, and by $\bbN_+$ the natural numbers without zero. $\calI$ is a set of indices with elements $i\in\bbN_+$, and $\vert\calI\vert$ gives the number of elements in the index set $\calI$.




\section{Problem statement}\label{sec:problem}
We consider 1D CNNs that are a concatenation of convolutional layers $\calC_i: \bbR^{c_{i-1} \times N_{i-1}} \to \bbR^{c_i \times N_i}$ with indices $i\in\calI_C$, and fully connected layers $\calL_i: \bbR^{n_{i-1}} \to \bbR^{n_{i}}$ with indices $i\in\calI_F$
\begin{align}\label{eq:CNN}
    \mathrm{CNN}_\theta = \calL_{l} \circ \ldots \circ \calL_{p+1} \circ F \circ \calC_{p} \circ \ldots \circ \calC_1,
\end{align}
adding up to a total number of $l=\vert\calI_C\vert+\vert\calI_F\vert$ layers. Here, $N_i$ denotes the signal length, $c_i$ the channel size, and $n_i$ the layer dimension of the respective $i$-th layer. To transition from the fully convolutional part of the CNN to the fully connected part, we necessitate a flattening operation $F:\bbR^{c_p\times N_p}\to\bbR^{n_p}$ that operates on the output of the $p$-th (last) convolutional layer with $n_p=c_pN_p$.

A \emph{convolutional} layer consists of two to three stages, a convolution operation, a nonlinear activation, and possibly a pooling operation. The first two stages are
\begin{equation}\label{eq:CNN_layer}
    \widetilde{\calC}_i: w_k^{i} = \phi_i\left(b_i + \sum_{j=0}^{\ell_i-1} K_j^i w_{k-j}^{i-1}\right), \quad k = 0,\ldots,N_i-1~\forall i\in\calI_C,
\end{equation}
with convolution kernel $K^i_j\in\bbR^{c_{i} \times c_{i-1}}$, $j=0,\dots,\ell_i-1$, kernel size $\ell_i$, and bias $b_i\in\bbR^{c_{i}}$. First, a convolution on the signal $w^{i-1}\in\bbR^{c_{i-1}\times N_{i-1}}$ is applied and subsequently, the nonlinear activation function $\phi_i: \bbR^{c_i} \to \bbR^{c_i}$ is evaluated element-wise to obtain the output $w^{i}\in\bbR^{c_{i}\times N_{i-1}}$. Oftentimes, a convolutional layer additionally contains pooling layers $\calP_i: \bbR^{c_{i} \times N_{i-1}} \to \bbR^{c_i \times N_i}$ to downsample the signal $w^{i}$. We consider maximum pooling 
\begin{align*}
    \calP_i^\mathrm{max} : \tilde{w}_k^i =\max_{j = 1,\ldots, \ell_i} w^{i}_{\ell_i (k - 1) + j},~k = 0,\ldots,N_i-1,\forall i\in\calI_P^\text{max},
\end{align*}
and average pooling
\begin{align*}
    \calP_i^\mathrm{av} : \tilde{w}_k^{i} = \frac{1}{\ell_i}\sum_{j = 1}^{\ell_i} w^i_{\ell_i (k - 1) + j},~k = 0,\ldots,N_i-1,\forall i\in\calI_P^\text{av},
\end{align*}
where $\calI_P^\text{av}\cup\calI_P^\text{max}\subseteq\calI_C$. As a result, the convolutional layer becomes  $\calC_i=\calP_i\circ\widetilde{\calC}_i$ in case a pooling layer is added or $\calC_i=\widetilde{\calC}_i$ otherwise. Finally, a CNN typically holds \emph{fully connected} layers, which we define as mappings
\begin{equation}\label{eq:FC_layer}
    \begin{split}
        \calL_i:~& w^i=\phi_i(W_{i}w^{i-1}+b_{i})\quad \forall i\in\calI_F\backslash \{l\},\\
        \calL_{l}:~& w^{l} = W_{l}w^{l-1}+b_{l}
    \end{split}
\end{equation}
with weights $W_i\in\bbR^{n_{i}\times n_{i-1}}$, biases $b_i\in\bbR^{n_{i}}$ and activation functions $\phi_i: \bbR^{n_i} \to \bbR^{n_i}$ that are applied element-wise.

The 1D CNN $f_\theta(w^0)=w^{l}$ is hence characterized by $\theta=\{(K^i,b_i)_{i=1}^{p},(W_i,b_i)_{i=p+1}^{l}\}$ and the chosen activation and pooling operations. In this work, we present a direct parameterization for Lipschitz-bounded 1D CNNs \eqref{eq:CNN}. 
\begin{problem}
    \label{problem2}
    Find a parameterization $\kappa\mapsto\theta$ of $f_\theta$ for a predefined Lipschitz bound $\rho> 0$ such that all 1D CNNs parameterized by $\kappa$ are $\rho$-Lipschitz continuous with respect to the $\ell_2$ norm, i.\,e., they satisfy 
    \begin{equation}\label{eq:lipschitz}
        \Vert f_\theta(x)-f_\theta(y)\Vert_2\leq \rho\Vert x-y \Vert_2\quad \forall x,y\in\bbR^n.
    \end{equation}
\end{problem}
In the case of multiple channels $c$, $n=cN$ denotes the stacked up version of the input. Note that $\|\cdot\|_2$ in \eqref{eq:lipschitz} can either be interpreted as the Euclidean norm of a vector-valued input $x$ or as the $\ell_2$ norm of a signal $x$.

To train a Lipschitz-bounded CNN, we minimize a learning objective $\calL(\theta)$, e.\,g., the mean squared error, the cross-entropy loss or, to encourage robustness through the learning objective a tailored loss, e.g. the hinge loss \cite{bethune2022pay}, while at the same time enforcing Lipschitz-boundedness onto the CNN. Rather than solving a training problem subject to a Lipschitz constraint, which can get computationally involved, i.\,e.,
\begin{equation*}
    \min_{\theta} ~ \calL(\theta) \quad \text{s.\,t.} \quad f_\theta~\text{is Lipschitz-bounded},
\end{equation*}
the suggested parameterization $\kappa\mapsto\theta$ allows to solve an unconstrained training problem over $\kappa$ using gradient methods
\begin{equation*}
    \min_{\kappa} ~ \calL(\theta(\kappa)).
\end{equation*}

\section{Preliminaries}\label{sec:prelims}
Before we state the parameterization of Lipschitz-bounded 1D CNNs in Section \ref{sec:parameterization}, we introduce a compact formulation of convolutions in state space and state LMI conditions that certify Lipschitz boundedness and that can be used to estimate the Lipschitz constant for 1D CNNs \cite{pauli2023lipschitz}. In addition, we introduce the Cayley transform used to parameterize orthogonal matrices.
\subsection{State space representation for convolutions}
To formulate LMI conditions for convolutional layers, we can either reformulate the convolutional operation as a fully connected layer characterized by a sparse and redundant Toeplitz matrix \cite{pauli2022neural} that scales with the input dimension or, as suggested in \cite{pauli2023lipschitz}, we can compactly state the convolution, i.\,e., an FIR filter, in state space, completely independent of the input signal length. A possible discrete-time state space representation of the $i$-th convolutional layer \eqref{eq:CNN_layer} with state $x_k^i\in\bbR^{n_{x_i}}$ and state dimension $n_{x_i}=(\ell_i-1)c_{i-1}$ is
\begin{equation}\label{eq:ss_conv}
    \begin{split}
        x_{k+1}^i &= A_i x_k^i + B_i w^{i-1}_k,\\ 
        y_k^i &= C_i x^i_k + D_i w^{i-1}_k + b_i,\\
        w_k^i &=\phi(y_k^i),
    \end{split}
\end{equation}
where
\begin{subequations}\label{eq:ABCD}
\begin{align}\label{eq:AB}
    &A_i =
    \begin{bmatrix}
        0 & I &  &  0 \\
        0  & 0 & \ddots &  \\
        \vdots &  & \ddots & I \\
        0 &    \dots   &   & 0 \\
    \end{bmatrix},
    &B_i &= \begin{bmatrix}
        0 \\
        \vdots \\
        0 \\
        I
    \end{bmatrix},\\\label{eq:CD}
    &C_i =
    \begin{bmatrix}
        K^i_{\ell_i-1} & \dots & K^i_1
    \end{bmatrix},
    &D_i &= K^i_0.
\end{align}
\end{subequations}
Note that 2D convolutions also admit a state space realization, namely as a 2D system \cite{gramlich2022convolutional}, based on which our parameterization can potentially be extended to end-to-end Lipschitz-bounded 2D CNNs.

\subsection{Lipschitz constant estimation}
The Lipschitz constant is a sensitivity measure to changes in the input, which is commonly used to verify robustness for NNs \cite{szegedy2013intriguing}. Since, however, the calculation of the true Lipschitz constant is an NP-hard problem, an accurate upper bound is sought instead. For this purpose, we over-approximate the nonlinear activation functions by their slope-restriction cone \cite{fazlyab2019efficient,pauli2021training}. Commonly used activation functions $\varphi:\bbR\to\bbR$, such as ReLU and tanh, are slope-restricted in $[0,1]$, i.\,e.,
\begin{equation*}
    0 \leq \frac{\varphi(x)-\varphi(y)}{x-y} \leq 1\quad\forall x,y\in\bbR.
\end{equation*}
Based on this \PP{element-wise} property, we formulate an incremental quadratic constraint 
    \begin{equation}\label{eq:slope_restriction}
        \begin{bmatrix}
            \phi(x)-\phi(y)\\
            x-y
        \end{bmatrix}^\top
        \begin{bmatrix}
            -2\Lambda & \Lambda\\
            \Lambda & 0 
        \end{bmatrix}
        \begin{bmatrix}
            \phi(x)-\phi(y)\\
            x-y
        \end{bmatrix}\geq 0~\forall x,y\in\bbR^{n},
    \end{equation}
\PP{which, by the multipliers $\Lambda\in\bbD_+^n$, is a conic combination of the element-wise slope-restriction constraint and a suitable over-approximation of the nonlinearities in all $n$ neurons}. The following theorem states a set of $l$ LMI conditions that serve as a sufficient condition for Lipschitz continuity for 1D CNNs based on the relaxation \eqref{eq:slope_restriction} \cite{pauli2023lipschitz}.
\begin{theorem} [\cite{pauli2023lipschitz}] \label{thm:certification}
    Let $\CNN_\theta$ and $\rho > 0$ be given and let all activation functions be slope-restricted in $[0,1]$. If there exist
    \begin{enumerate}[label=(\roman*)]
        \item  $Q_i\in\bbS^{c_i}$ ($Q_i\in\bbD^{c_i}$ if a convolutional layer contains a maximum pooling layer), $P_i\in\bbS_+^{n_{x_i}}$, and $\Lambda_i\in\bbD_+^{c_i}$ such that $\forall i\in\calI_C$
        \begin{align}\label{eq:LMI-CNN-layers}
            \begin{split}
            \left[\begin{array}{cc|c}
            P_i-A_i^\top P_iA_i  & -A_i^\top P_iB_i & -C_i^\top\Lambda_i\\
            -B_i^\top P_iA_i &  Q_{i-1}-B_i^\top P_iB_i  & -D_i^\top\Lambda_i\\\hline
            -\Lambda_i C_i & -\Lambda_i D_i &2\Lambda_i-Q_i
            \end{array}\right]\succeq 0,
            \end{split}
        \end{align}
    where $Q_0=\tilde{\rho}^2 I$,
        \item $Q_i\in\bbS^{n_i}$ and $\Lambda_i\in\bbD_+^{n_i}$ such that $\forall i = \calI_F\backslash \{l\}$
\begin{align}\label{eq:LMI_FC_layers}
    \begin{split}
    \begin{bmatrix}
        Q_{i-1} & -W_{i}^\top \Lambda_{i}\\
        -\Lambda_{i} W_{i} & 2\Lambda_{i}-Q_{i}
    \end{bmatrix}\succeq 0~ \text{and}~    
    \begin{bmatrix}
        Q_{l-1} & -W_{l}^\top\\
        -W_{l} & I
    \end{bmatrix}\succeq 0,
    \end{split}
\end{align}
    where $Q_p := I_{N_p} \otimes Q_p$,
    \end{enumerate}
    then the $\CNN_\theta$ is $\rho$-Lipschitz continuous with $\rho = \tilde{\rho} \prod_{s\in\calI_{P}^\text{av}} \mu_{s}$, where $\mu_s$ are the Lipschitz constants of the average pooling layers.
\end{theorem}

The proof of Theorem \ref{thm:certification} is based on the dissipativity of the individual layers of $\CNN_\theta$ that are connected in a feedforward fashion, see \cite{pauli2023lipschitz} for details and the proof. Note that the matrix $Q_{i-1}$, which is a gain matrix, links the $i$-th layer to the previous layer and by this interconnection we can finally analyze Lipschitz continuity of the input-output mapping $w^l=\CNN_\theta(w^0)$ \cite{pauli2023lipschitz}.

Based on Theorem \ref{thm:certification}, we can determine an upper bound on the Lipschitz constant for a given CNN solving a semidefinite program
\begin{equation}\label{eq:Lipschitz_estimation}
    \min_{\rho^2,\Lambda,P,Q} ~\rho^2 \quad \text{s.\,t.} \quad \eqref{eq:LMI-CNN-layers}, \eqref{eq:LMI_FC_layers},
\end{equation}
where $\Lambda=\{\Lambda_i\}_{i\in\calI_C\cup \calI_F\backslash \{l\}}$, $Q=\{Q_i\}_{i\in\calI_C\cup \calI_F\backslash \{l\}}$, $P=\{P_i\}_{i\in\calI_C}$ serve as decision variables together with $\rho^2$. 

\subsection{Cayley transform}
Typically, the Cayley transform maps skew-symmetric matrices to orthogonal matrices and its extended version parameterizes the Stiefel manifold from non-square matrices, which can be useful in designing NNs \cite{wang2023direct,trockman2021orthogonalizing,helfrich2018orthogonal}. 
\begin{lemma}[Cayley transform \cite{jauch2020random}]\label{lem:Cayley}
    For all $Y\in\bbR^{n\times n}$ and $Z\in\bbR^{m\times n}$ the Cayley transform
\begin{equation*}\label{eq:Cayley}
    \begin{split}
        \Cayley\left(
    \begin{bmatrix}
        Y \\
        Z
    \end{bmatrix}\right)=
    \begin{bmatrix}
        U \\
        V
    \end{bmatrix}=
    \begin{bmatrix}
        (I+M)^{-1}(I-M) \\
        2Z(I+M)^{-1}
    \end{bmatrix},
    \end{split}
\end{equation*}
where $M=Y-Y^\top+Z^\top Z$, yields matrices $U\in\bbR^{n\times n}$ and $V\in\bbR^{m\times n}$ that satisfy $U^\top U + V^\top V =  I$.
\end{lemma}
Note that $I+M$ is nonsingular since $1\leq\lambda_\mathrm{min}(I+Z^\top Z)\leq Re(\lambda_\mathrm{min}(I+M))$. 

\section{Direct parameterization}\label{sec:parameterization}
While \eqref{eq:Lipschitz_estimation} analyzes Lipschitz continuity for given 1D CNNs, it also is desirable to train robust CNNs, i.\,e., $\rho$-Lipschitz bounded CNNs where the robustness level $\rho$ is chosen by the user. 
In this section, we introduce a layer-wise parameterization for 1D CNNs \eqref{eq:CNN} that renders the input-output mapping Lipschitz continuous. We first discuss a parameterization for fully connected layers that satisfy \eqref{eq:LMI_FC_layers} by design, using a similar construction to \cite{wang2023direct}. Our key contribution then is the parameterization of convolutional layers, which is carried out in two steps. In a first step, we establish a parameterization of $P_i$ that renders the left upper block in \eqref{eq:LMI-CNN-layers} positive definite using the controllability Gramian and afterwards, we introduce the parameterization for convolutional layers that by design satisfy \eqref{eq:LMI-CNN-layers}.
\subsection{Fully connected layers}
In the following, we present a mapping $\kappa_i\mapsto(W_i,b_i)$ from unconstrained variables $\kappa_i$ that renders \eqref{eq:LMI_FC_layers} feasible by design.
\begin{theorem}
    Fully connected layers \eqref{eq:FC_layer} parameterized by
    \begin{subequations}\label{eq:parameterization_FNN_all}
        \begin{align}\label{eq:parameterization_FNN}
            W_i &=  \sqrt{2}\Gamma_i^{-1}V_i^\top L_{i-1}, & b_i\in\bbR^{n_i}, & \quad \forall i\in\calI_F\backslash\{l\},\\\label{eq:parameterization_FNN_last}
            W_l &=  V_l^\top L_{l-1}, & b_l\in\bbR^{n_l}, &
        \end{align}
    \end{subequations}
wherein
    \begin{equation*}
        \Gamma_i=\diag(\gamma_i),~
        L_i=\sqrt{2} U_{i}\Gamma_{i},~
        \begin{bmatrix}
            U_i \\
            V_i
        \end{bmatrix}=
        \Cayley\left(
        \begin{bmatrix}
            Y_i\\
            Z_i
        \end{bmatrix}
        \right),
    \end{equation*}
    satisfy \eqref{eq:LMI_FC_layers}. This yields the mappings $(Y_i,Z_i,\gamma_i,b_i)\mapsto(W_i,b_i)$, $i\in\calI_F\backslash\{l\}$, and $(Y_l,Z_l,b_l)\mapsto(W_l,b_l)$, respectively, where $Y_i\in\bbR^{n_i\times n_i}$, $Z_i\in\bbR^{n_{i-1}\times n_i}$, $b_i\in\bbR^{n_i}$, $i\in\calI_F$, $\gamma_i\in\bbR^{n_i}$, $i\in\calI_F\backslash\{l\}$ are free variables.
\end{theorem}
\begin{proof}
According to Lemma \ref{lem:Cayley}, $U_i$ and $V_i$ satisfy $U_i^\top U_i + V_i^\top V_i =  I$, wherein we insert the parameterization \eqref{eq:parameterization_FNN} and $U_i=\frac{1}{\sqrt{2}}L_i\Gamma_i^{-1}$ to obtain 
\begin{equation*}
    \frac{1}{2}(\Gamma_i^{-1}L_i^\top L_i \Gamma_i^{-1} + \Gamma_i W_i^\top L_{i-1}^{-1}  L_{i-1}^{-\top} W_i \Gamma_i) =  I.
\end{equation*}
With $Q_i=L_i^\top L_i$ and $\Lambda_i=\Gamma_i^\top\Gamma_i$, we further obtain
\begin{equation*}
    Q_i + \Lambda_i W_i Q_{i-1}^{-1} W_i^\top \Lambda_i =  2\Lambda_i,
\end{equation*}
which implies $2\Lambda_i - Q_i - \Lambda_i W_i Q_{i-1}^{-1} W_i^\top \Lambda_i \succeq 0$. 
Next, we apply the Schur complement, which yields the left inequality in \eqref{eq:LMI_FC_layers}.
The last fully connected layer is a special case that does not contain an activation function. Inserting the parameterization \eqref{eq:parameterization_FNN_last} gives
\begin{equation*}
    U_l^\top U_l + V_l^\top V_l = U_l^\top U_l + W_l Q_{l-1}^{-1}W_l^\top=I,
\end{equation*}
which implies $I-W_l Q_{l-1}^{-1}W_l^\top= U_l^\top U_l \succeq 0$, which by the Schur complement satisfies the right inequality in \eqref{eq:LMI_FC_layers}.
\end{proof}
Note that the connection between the auxiliary matrices $L_i$ in \eqref{eq:parameterization_FNN_all} and $Q_i$ in \eqref{eq:LMI_FC_layers} is $L_i^\top L_i=Q_i$ and the relation between the multiplier matrices $\Lambda_i$ in \eqref{eq:slope_restriction} / \eqref{eq:LMI_FC_layers} and $\Gamma_i$ in \eqref{eq:parameterization_FNN_all} is $\Gamma_i^\top\Gamma_i=\Lambda_i$.
\begin{remark}
     Throughout the paper, we assume that $\Gamma_i$ and $L_i$ are nonsingular. In our experiments, this was always the case. However, there also are tricks to enforce this property, e.g., by choosing $\Gamma_i=\diag(e^{\gamma_i})$ \cite{wang2023direct}.
\end{remark}
\begin{remark}
    Our parameterization \eqref{eq:parameterization_FNN_all} is equivalent to the one established in \cite{wang2023direct}, where they show that it is necessary and sufficient, i.\,e., the fully connected layers \eqref{eq:FC_layer} satisfy \eqref{eq:LMI_FC_layers} if and only if the weights can be parameterized by \eqref{eq:parameterization_FNN_all}.
\end{remark}

\subsection{Parameterization by the controllability Gramian}
In this section, we make use of the controllability Gramian of \eqref{eq:ss_conv} to parameterize convolutional layers, which to the best knowledge of the authors, has thus far not appeared in the context of parameterizing NNs. For that purpose, we introduce
\begin{equation}\label{eq:matrixF}
F_i :=
\left[\begin{array}{cc}
    P_i-A_i^\top P_iA_i  & -A_i^\top P_iB_i\\
    -B_i^\top P_iA_i &  Q_{i-1}-B_i^\top P_iB_i 
\end{array}\right]\succ 0,
\end{equation}
which is the left upper block in \eqref{eq:LMI-CNN-layers} and further, we introduce $\widehat{C}_i :=
        \begin{bmatrix}
            C_i & D_i
        \end{bmatrix}$, which simplifies the notation of \eqref{eq:LMI-CNN-layers} to
\begin{equation}\label{eq:LMI-CNN-layers2}
\begin{bmatrix}
    F_{i} & -\widehat{C}_i^\top\Lambda_i\\
    -\Lambda_i\widehat{C}_i & 2\Lambda_i-Q_i
\end{bmatrix}\succeq 0.
\end{equation}
        
We note that the LMI \eqref{eq:LMI-CNN-layers2} and the left LMI in \eqref{eq:LMI_FC_layers} share a similar structure. The right lower block is the same in both LMIs. Beside the biases $b_i$, the parameters to be trained in the CNN layers are collected in $\widehat{C}_i$, cmp. \eqref{eq:ABCD}, whereas the parameters $W_i$ characterize the fully connected layers. In the off-diagonal blocks of the respective LMIs \eqref{eq:LMI-CNN-layers2} and \eqref{eq:LMI_FC_layers} $\widehat{C}_i$ and $W_i$ appear respectively multiplied by $\Lambda_i$. The only difference appears in the left upper  blocks of the LMIs. While in LMI \eqref{eq:LMI_FC_layers} for fully connected layers, we here have $Q_{i-1}=L_{i-1}^\top L_{i-1}\succ 0$ for nonsingular $L_{i-1}$, 
LMI \eqref{eq:LMI-CNN-layers2} for convolutional layers here contains $F_i$, that depends on $Q_{i-1}$. To render $F_i$ positive definite, we parameterize $P_i$ as follows, using the controllability Gramian.
\begin{lemma}\label{lem:Gramian}
For some $\varepsilon>0$ and all $H_i\in\bbR^{n_{x_i}\times n_{x_i}}$, the matrix $P_i=X_i^{-1}$ with
\begin{equation}\label{eq:Gramian}
    X_i = \sum_{k=0}^{n_{x_i}-c_{i-1}} A_i^k(B_i{Q}_{i-1}^{-1}B_i^\top+H_i^\top H_i+\epsilon I)(A_i^\top)^{k},
\end{equation}
renders \eqref{eq:matrixF} feasible.
\end{lemma}
\begin{proof}
The matrix $A_i$ is a nilpotent matrix, i.\,e., $A^{n_{x_i}-c_{i-1}+k}=0$ $\forall k\geq1$, such that \eqref{eq:Gramian} can be written as
\begin{equation*}
    X_i = \sum_{k=0}^{\infty} A_i^k(B_i{Q}_{i-1}^{-1}B_i^\top+H_i^\top H_i+\epsilon I)(A_i^\top)^{k},
\end{equation*}
which corresponds to the controllability Gramian of the linear time-invariant system characterized by $(A_i,B_i)$ as defined in \eqref{eq:AB}, i.\,e., the unique solution $X_i\succ 0$ to the Lyapunov equation
\begin{equation}\label{eq:Lyapunov}
    X_i-A_iX_iA_i^\top-B_iQ_{i-1}^{-1}B_i^\top=H_i^\top H_i+\epsilon I\succ 0.
\end{equation}
Note that $X_i$ is positive definite by design, given that $Q_{i-1}=L_{i-1}^\top L_{i-1}\succ 0$ such that $B_i{Q}_{i-1}^{-1}B_i^\top+H_i^\top H_i+\epsilon I$ is positive definite. Next, we apply the Schur complement to \eqref{eq:Lyapunov} to obtain
\begin{equation*}
\begin{bmatrix}
    X_i^{-1}  & 0         & A_i^\top\\
    0         & Q_{i-1}   & B_i^\top \\
    A_i       & B_i       & X_i
\end{bmatrix}\succ 0.
\end{equation*}
Now inserting $P_i = X_i^{-1}$ and again applying the Schur complement yields \eqref{eq:matrixF}.
\end{proof}


\subsection{Convolutional layers}
In this subsection, we present a direct parameterization for convolutional layers such that they satisfy \eqref{eq:LMI-CNN-layers2} by design. Our parameterization of convolution kernels $K^i\in\bbR^{c_i\times c_{i-1}\times \ell_i}$, or equivalently $\widehat{C}_i\in\bbR^{c_i\times \ell_ic_{i-1}}$, is independent of the input dimension $N_i$ whereas other approaches design a parameterization for Lipschitz-bounded convolutions and CNNs in the Fourier domain which involves the inversion of $N_i$ matrices \cite{trockman2021orthogonalizing,wang2023direct}.
\begin{theorem}
    Convolutional layers \eqref{eq:CNN_layer} that contain an average pooling layer or no pooling layer parameterized by
    \begin{equation}\label{eq:parameterization_CNN}
        \widehat{C}_i
        =  \sqrt{2}\Gamma_i^{-1}V_i^\top L^F_{i},~b_i\in\bbR^{c_i},~\forall i\in\calI_C\backslash \calI_C^{\mathrm{max}},
    \end{equation}
wherein
\begin{equation*}
    \Gamma_i=\diag(\gamma_i),~
    \begin{bmatrix}
        U_i \\
        V_i
    \end{bmatrix}=
    \Cayley\left(
    \begin{bmatrix}
        Y_i\\
        Z_i
    \end{bmatrix}
    \right),~
    L_i^F=\chol(F_i),
\end{equation*}
satisfy \eqref{eq:LMI-CNN-layers}. Here, $\chol(\cdot)$ denotes the Cholesky decomposition, $Q_i=L_i^\top L_i,~L_0 =\rho I,~L_i=\sqrt{2} U_{i}\Gamma_{i}$, $F_i$ is given by \eqref{eq:matrixF} with $P_i$ parameterized from $Q_{i-1}$ and $H_i$ using \eqref{eq:Gramian}. The free variables beside $b_i$ are $Y_i\in\bbR^{c_{i}\times c_{i}}$, $Z_i\in\bbR^{\ell_i c_{i-1}\times c_{i}}$, $H_i\in\bbR^{n_{x_i}\times n_{x_i}}$, $\gamma_i\in\bbR^{c_{i}}$, $i\in\calI_C\backslash \calI_C^{\mathrm{max}}$, which yields the mapping $(Y_i,Z_i,H_i,\gamma_i,b_i)\mapsto(K^i,b_i)$.
\end{theorem}
\begin{proof}
The matrices $U_i$ and $V_i$ satisfy $U_i^\top U_i + V_i^\top V_i =  I$, wherein we insert the parameterization \eqref{eq:parameterization_CNN} to obtain
\begin{equation*}
    \frac{1}{2}(\Gamma_i^{-1}L_i^\top L_i \Gamma_i^{-1} + \Gamma_i W_i^\top {L^F_{i-1}}^{-1}  {L^F_{i-1}}^{-\top} W_i \Gamma_i) =  I.
\end{equation*}
Lemma \ref{lem:Gramian} ensures positive definiteness of $F_i$, i.\,e., its Cholesky decomposition exists, and we insert $Q_i=L_i^\top L_i$, $F_i={L_i^F}^\top L_i^F$, $\Lambda_i=\Gamma_i^\top\Gamma_i$, to further obtain
\begin{equation*}
    Q_i + \Lambda_i \widehat{C}_i F_{i}^{-1} \widehat{C}_i^\top \Lambda_i =  2\Lambda_i,
\end{equation*}
which implies $2\Lambda_i - Q_i - \Lambda_i \widehat{C}_i F_{i}^{-1} \widehat{C}_i^\top \Lambda_i \succeq 0$. Next, we apply the Schur complement and obtain \eqref{eq:LMI-CNN-layers2} which corresponds to \eqref{eq:LMI-CNN-layers}.
\end{proof}
To account for average pooling layers present in the CNN, we rescale the desired Lipschitz bound with the product of the Lipschitz bounds of the average pooling layers, i.\,e.,  $\tilde{\rho}=\rho/\Pi_{s\in\calI_P^\text{av}}$, cmp. Theorem \ref{thm:certification}, and then we use \eqref{eq:parameterization_CNN} to parameterize the convolutional layer.

\subsection{Maximum pooling layers}
In case there is a maximum pooling layer in $\calC_i$, i.\,e. $\calC_i=\calP_i^\mathrm{max}\circ \widetilde{\calC}_i$, we need to parameterize $Q_i$ as a diagonal matrix, cmp. Theorem \ref{thm:certification}, which also affects the parameterization. This constraint comes from the incremental quadratic constraint used to include maximum pooling layers \cite{pauli2023lipschitz}. While average pooling is a linear operation and allows for full matrices $Q_i$, the nonlinearity of the maximum pooling layer is the reason for the additional diagonality constraint.
\begin{corollary}\label{cor:maxpooling}
    Convolutional layers \eqref{eq:FC_layer} that contain a maximum pooling layer parameterized by
    \begin{equation}\label{eq:parameterization_CNN2}
        \widehat{C}_i =  \Lambda_i^{-1}\widetilde{\Gamma}_i\widetilde{U}_i^\top L^F_{i},~b_i\in\bbR^{c_i},~\forall \calI_C^\mathrm{max}
    \end{equation}
satisfy \eqref{eq:LMI-CNN-layers}, wherein
\begin{align*}
    \Lambda_i = \frac{1}{2}\left(\widetilde{\Gamma}_i^\top\widetilde{\Gamma}_i+Q_i\right),~
    \widetilde{\Gamma}_i\!=\!\diag(\tilde{\gamma}_i),~
    \widetilde{U}_i=\Cayley(\widetilde{Y}_i),
\end{align*}
$Q_i=L_i^\top L_i,~L_0 =\rho I,~L_i =\diag(l_i)$, $L_i^F=\chol(F_i)$, where $F_i$ is given by \eqref{eq:matrixF} with $P_i$ parameterized from $Q_{i-1}$ and $H_i$ using \eqref{eq:Gramian}. The free variables $\widetilde{Y}_i\in\bbR^{\ell_ic_{i-1}\times c_i}$, $H_i\in\bbR^{n_{x_i}\times n_{x_i}}$, $\tilde{\gamma}_i,l_i\in\bbR^{c_{i}}$, $i\in\calI_C^\mathrm{max}$ compose the mapping $(\widetilde{Y}_i,H_i,\tilde{\gamma}_i,l_i,b_i)\mapsto(K^i,b_i)$.
\end{corollary}
\begin{proof}
Using that $\widetilde{U}_i$ satisfies $\widetilde{U}_i^\top\widetilde{U}_i=  I$, we insert \eqref{eq:parameterization_CNN2} and we replace $F_i={L_i^F}^\top L_i^F$, whose Cholseky decomposition exists by Lemma~\ref{lem:Gramian}, to obtain
\begin{equation*}
    \Lambda_i\widehat{C}_i F^{-1} \widehat{C}_i^\top \Lambda_i = \widetilde{\Gamma}_i^\top\widetilde{\Gamma}_i =   2\Lambda_i-Q_i,
\end{equation*}
which implies $2\Lambda_i-Q_i-\Lambda_i\widehat{C}_i F^{-1} \widehat{C}_i^\top \Lambda_i\succeq 0$, which by the application of the Schur complement satisfies \eqref{eq:LMI-CNN-layers2}.
\end{proof}

\subsection{$\rho$-Lipschitz layers}
In this work, we enforce Lipschitz continuity onto the input-ouput mapping of the CNN, which is more general and less conservative than approaches that include Lipschitz guarantees for the individual layers, using that the product of the Lipschitz bounds gives the Lipschitz bound for the NN \cite{araujo2023a}. 
In this section, we adapt our approach to the special case of parameterizing $\rho$-Lipschitz linear layers, i.\,e., by choosing $Q_{i-1}= \rho^2 I$, $Q_i=I$, and $\Lambda_i = I$. 


\begin{corollary}\label{cor:FC_layer}
    The $i$-th linear fully connected layer
    \begin{equation*}
        v^i = W_{i}w^{i-1}+b_{i} \quad \text{with}
        \quad  W_i = \rho\widetilde{U}_i^\top,~
        b_i\in\bbR^{n_i},
    \end{equation*}
    is $\rho$-Lipschitz continuous, where $\widetilde{U}_i=\Cayley(\widetilde{Y}_i)$, and $\widetilde{Y}_i\in\bbR^{n_{i-1}\times n_i}$ is a free variable, which yields the mapping $(\widetilde{Y}_i,b_i) \mapsto (W_i,b_i)$.
\end{corollary}


\begin{corollary}\label{cor:conv_layer}
    The $i$-th 1D convolution
    \begin{equation*}
        v_i = b_i + \sum_{j=0}^{\ell_i-1} K_j^i w_{k-j}^{i-1}\quad\text{with}\quad \widehat{C}_i  =  \widetilde{U}_i^\top {L_i^F},~b_i\in\bbR^{c_i},
    \end{equation*}
    where $K^i$ is recovered according to \eqref{eq:CD}, is $\rho$-Lipschitz continuous. Herein, $\widetilde{U}_i=\Cayley(\widetilde{Y}_i)$, $L_i^F=\chol(F_i)$, where $F_i$ is given by \eqref{eq:matrixF} with $P_i$ parameterized from $Q_{i-1}=\rho^2 I$ and $H_i$ using \eqref{eq:Gramian}. Beside $b_i$ the free variables are $\widetilde{Y}_i\in\bbR^{\ell_ic_{i-1}\times c_i}$, $H_i\in\bbR^{n_{x_i}\times n_{x_i}}$, yielding the mapping $(\widetilde{Y}_i,H_i,b_i)\mapsto(K^i,b_i)$.
\end{corollary}
The proofs of Corollaries \ref{cor:FC_layer} and \ref{cor:conv_layer} follow along the lines of the proof of Corollary~\ref{cor:maxpooling}.

\section{Simulation results}\label{sec:simulation}
\begin{figure}
\centering
\input{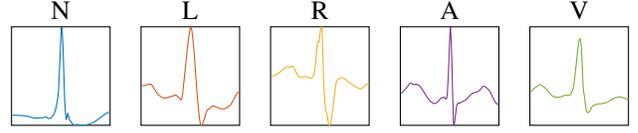}
\caption{Different heart wave classes.}
\label{fig:heartwaves}
\vspace{-0.1cm}
\end{figure}

In this section, we train Lipschitz-bounded 1D CNNs (LipCNNs) to classify heart arrythmia data from the MIT-BIH database, an ECG database \cite{mitbihdatabase}, in its preprocessed form \cite{abuadbba2020git} that assigns heart wave signals to one of the five classes depicted in Fig. \ref{fig:heartwaves}: N (normal beat), L (left bundle branch block), R (right bundle branch block), A (atrial premature contraction), V (ventricular premature contraction). In particular, we use 26,490 samples, 13,245 training data points and 13,245 test data points. Fig.~\ref{fig:CNN} shows the architecture of the 1D CNN used to train on the ECG dataset, and for training we choose the cross-entropy loss as the objective. 
For training, we use Flux, JuMP, and RobustNeuralNetworks.jl \cite{barbara2023robustneuralnetworks} in Julia in combination with MOSEK \cite{mosek} on a standard i7 notebook\footnote{Code is available at \url{www.github.com/ppauli/1D-LipCNNs}.}.

We compare our approach to vanilla CNNs and L2 regularized CNNs with different weighting factors $\gamma$. To evaluate the robustness of the CNNs, we compute the test accuracy on adversarial examples from the L2 projected gradient descent attack for different perturbation strengths, which is shown in Fig. \ref{fig:PGD}. Note that the two shown LipCNNs indicate comparably high test accuracies to the shown vanilla and L2 regularized CNNs on unperturbed data while maintaining higher accuracies than their counterparts as the perturbation strength increases. Further, Table~\ref{tab{acc+Lip}} shows the averaged test accuracies and upper and lower Lipschitz bounds for different 1D CNNs and   Fig.~\ref{fig:robacc} illustrates the tradeoff between accuracy and robustness. Low Lipschitz bounds mean high robustness, yet we observe lower test accuracies and vice versa. We note that the vanilla CNN has significantly larger upper Lipschitz bounds than LipCNN and further, LipCNN maintains high test accuracies as the Lipschitz bound decreases in Fig.~\ref{fig:robacc}. With even larger weighting parameters $\gamma$ in L2 regularized training the training failed, whereas LipCNNs allows for training with very low Lipschitz bounds.

\begin{figure}
\begin{subfigure}[b]{0.22\textwidth}
    \hspace{-0.6cm}    
    \begin{minipage}{\textwidth}
    {\tiny
    \begin{tikzpicture}[node distance = 0.55cm, auto]
        \node [block4] (input) {
        Input signal: 128 (dimension) x 1 (channel)
        };
        \node [block, below of=input] (CNN1) {
        Conv: 3 kernel + front padding: 128 x 1
        };
        \node [block2, below of=CNN1,node distance = 0.55cm] (avpool1) {
        Average pool: 2 kernel + 2 stride: 128 x 2
        };
        \node [block, below of=avpool1] (CNN2) {
        Conv: 3 kernel + front padding: 64 x 3
        };
        \node [block2, below of=CNN2,node distance = 0.55cm] (avpool2) {
        Average pool: 2 kernel + 2 stride: 32 x 3
        };
        \node [block3, below of=avpool2,node distance = 0.55cm] (FC1) {
        Dense: 60 fully connected layer
        };
        \node [block3, below of=FC1] (FC2) {
        Dense: 5 fully connected layer
        };
        \node [block4, below of=FC2] (output) {
        Output: 1 of 5 classes
        };
        \path [line] (input) -- (CNN1);
        \path [line] (CNN1) -- node [midway,align=right] {ReLU} (avpool1);
        \path [line] (avpool1) -- (CNN2);
        \path [line] (CNN2) -- node [midway,align=right] {ReLU} (avpool2);
        \path [line] (avpool2) -- node [midway,align=right] {flatten} (FC1);
        \path [line] (FC1) -- node [midway,align=right] {ReLU} (FC2);
        \path [line] (FC2) -- (output);
    \end{tikzpicture}}
\end{minipage}
\caption{1D CNN architecture.}\label{fig:CNN}
\end{subfigure}
\begin{subfigure}[b]{0.26\textwidth}
\centering
%
%
\definecolor{mycolor1}{rgb}{0.00000,0.44700,0.74100}%
\definecolor{mycolor2}{rgb}{0.85000,0.32500,0.09800}%
\definecolor{mycolor3}{rgb}{0.92900,0.69400,0.12500}%
\definecolor{mycolor4}{rgb}{0.49400,0.18400,0.55600}%
\begin{tikzpicture}

\begin{axis}[%
width=1.2in,
height=1.2in,
at={(0.37in,0.344in)},
scale only axis,
xmin=0,
xmax=12.3,
xlabel style={font=\color{white!15!black}},
xlabel={\small{perturbation strength}},
ymin=0.4,
ymax=1,
ylabel style={font=\color{white!15!black}},
xlabel near ticks,
ylabel={\small{test accuracy}},
ylabel near ticks,
axis background/.style={fill=white},
legend style={at={(0.03,0.03)}, anchor=south west, legend cell align=left, align=left, draw=white!15!black},
every tick label/.append style={font=\footnotesize}
]
\addplot [color=mycolor1, line width=1.0pt]
  table[row sep=crcr]{%
0	0.954699886749717\\
0.6	0.950245375613439\\
1.2	0.944884862212156\\
1.8	0.940505851264628\\
2.4	0.934692336730842\\
3	0.92638731596829\\
3.6	0.91521328803322\\
4.2	0.90554926387316\\
4.8	0.893091732729332\\
5.4	0.878822197055493\\
6	0.862665156662892\\
6.6	0.844998112495281\\
7.2	0.824160060400151\\
7.8	0.795696489241223\\
8.4	0.762853907134768\\
9	0.725179312948282\\
9.6	0.681766704416761\\
10.2	0.63035107587769\\
10.8	0.58233295583239\\
11.4	0.533484333710834\\
12.6	0.444469611174028\\
13.2	0.401132502831257\\
13.8	0.366628916572291\\
14.4	0.32955832389581\\
15	0.301321253303133\\
};
\addlegendentry{\footnotesize{vanilla}}

\addplot [color=mycolor2, line width=1.0pt]
  table[row sep=crcr]{%
0	0.911060777651944\\
0.6	0.906983767459419\\
1.2	0.902604756511891\\
1.8	0.897772744431861\\
2.4	0.892563231408078\\
3	0.885088712721782\\
3.6	0.877085692714232\\
4.2	0.86719516798792\\
4.8	0.858361645904115\\
5.4	0.848244620611551\\
6	0.838052095130238\\
6.6	0.825519063797659\\
7.2	0.812080030200075\\
7.8	0.798112495281238\\
8.4	0.782030955077388\\
9	0.76481691204228\\
9.6	0.74563986409966\\
10.2	0.725707814269536\\
10.8	0.70245375613439\\
11.4	0.678067195167988\\
12.6	0.62551906379766\\
13.2	0.597508493771234\\
13.8	0.572668931672329\\
14.4	0.548357870894677\\
15	0.525330313325783\\
};
\addlegendentry{\footnotesize{L2$_{\gamma=0.1}$}}

\addplot [color=mycolor3, line width=1.0pt]
  table[row sep=crcr]{%
0	0.90887127217818\\
0.6	0.903057757644394\\
1.2	0.897017742544356\\
1.8	0.891430728576821\\
2.4	0.885617214043035\\
3	0.880634201585504\\
3.6	0.87361268403171\\
4.2	0.866893167232918\\
4.8	0.861532653831635\\
5.4	0.854964137410343\\
6	0.847338618346546\\
6.6	0.839637599093998\\
7.2	0.832540581351453\\
7.8	0.822650056625142\\
8.4	0.812231030577576\\
9	0.801132502831257\\
9.6	0.78875047187618\\
10.2	0.776217440543601\\
10.8	0.76315590788977\\
11.4	0.745262363155908\\
12.6	0.710683276708192\\
13.2	0.695205738014345\\
13.8	0.679199697999245\\
14.4	0.664930162325406\\
15	0.649452623631559\\
};
\addlegendentry{\footnotesize{Lip10}}

\addplot [color=mycolor4, line width=1.0pt]
  table[row sep=crcr]{%
0	0.961192902982257\\
0.6	0.957417893544734\\
1.2	0.953416383540959\\
1.8	0.948886372215931\\
2.4	0.942770856927142\\
3	0.935447338618347\\
3.6	0.929784824462061\\
4.2	0.920724801812005\\
4.8	0.910230275575689\\
5.4	0.900490751226878\\
6	0.889467723669309\\
6.6	0.875877689694224\\
7.2	0.861079652699132\\
7.8	0.84484711211778\\
8.4	0.824311060777652\\
9	0.802567006417516\\
9.6	0.78082295205738\\
10.2	0.755832389580974\\
10.8	0.729709324273311\\
11.4	0.696338240845602\\
12.6	0.626651566628917\\
13.2	0.589203473008682\\
13.8	0.551377878444696\\
14.4	0.509324273310683\\
15	0.466364665911665\\
};
\addlegendentry{\footnotesize{Lip50}}

\end{axis}
\end{tikzpicture}%
\caption{Adversarial attack.}
\label{fig:PGD}
\end{subfigure}
\caption{(a) CNN used for ECG dataset. (b) Vanilla, L2 regularized ($\gamma= 0.1$) and Lipschitz ($\rho = 10, 50$)  CNN on adversarial examples with different perturbation strengths $\epsilon$ created using the L2 projected gradient descent attack \cite{madry2017towards}.}
\end{figure}
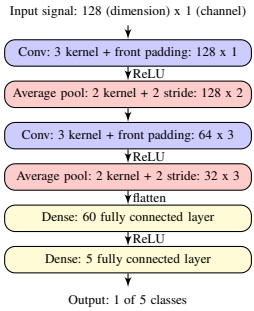
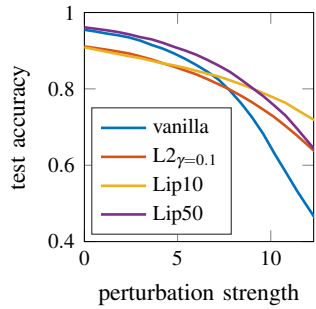

\begin{table}
\caption{Test accuracy and Lipschitz upper bound \cite{pauli2023lipschitz}  and empirical Lipschitz lower bound of trained CNNs, averaged over 5 CNNs.}
\label{tab{acc+Lip}}
\begin{center}
\begin{tabular}{ c | c c c c c c}
    & vanilla & L2$_{\gamma = 0.05}$  & L2$_{\gamma = 0.1}$ & Lip5 & Lip10 & Lip50 \\ \hline
 Test acc.   & $\!94.9\%\!$ & $\!92.4\%\!$ & $\!90.6\%\!$  & $\!84.8\%\!$ & $\!90.0\%\!$ & $\!94.7\%\!$ \\
 Lip. UB & 147 & 45.7 &  33.9 & 4.99 & 9.85 & 45.3\\
 Emp. LB & 28.3 &  13.2 & 10.2 & 2.20   & 4.06 & 15.1     
\end{tabular}
\end{center}
\vspace{-0.3cm}
\end{table}

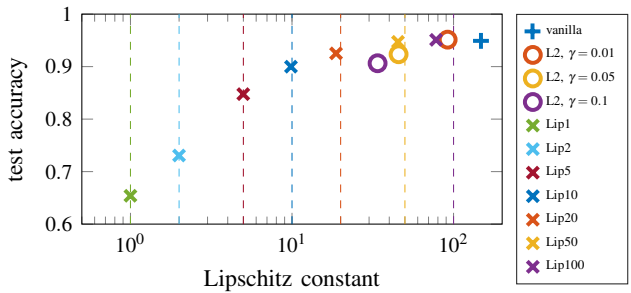
\begin{figure}
\centering
%
%
\definecolor{mycolor1}{rgb}{0.00000,0.44700,0.74100}%
\definecolor{mycolor2}{rgb}{0.85000,0.32500,0.09800}%
\definecolor{mycolor3}{rgb}{0.92900,0.69400,0.12500}%
\definecolor{mycolor4}{rgb}{0.49400,0.18400,0.55600}%
\definecolor{mycolor5}{rgb}{0.46600,0.67400,0.18800}%
\definecolor{mycolor6}{rgb}{0.30100,0.74500,0.93300}%
\definecolor{mycolor7}{rgb}{0.63500,0.07800,0.18400}%
\begin{tikzpicture}

\begin{axis}[%
width=2.2in,
height=1.1in,
at={(0.406in,0.344in)},
scale only axis,
xmode=log,
xmin=0.5,
xmax=200,
xminorticks=true,
ymin=0.6,
ymax=1,
xlabel={\small{Lipschitz constant}},
xlabel near ticks,
ylabel={\small{test accuracy}},
ylabel near ticks,
axis background/.style={fill=white},
legend style={at={(1.3,0.357)}, anchor=east, legend cell align=left, align=left, draw=white!15!black},
every tick label/.append style={font=\footnotesize}
]
\addplot [color=mycolor1, line width=1.5pt, only marks, mark size=3.0pt, mark=+, mark options={solid, mycolor1}]
  table[row sep=crcr]{%
147.82996888	0.9490524726\\
};
\addlegendentry{\tiny{vanilla}}

\addplot [color=mycolor2, line width=1.5pt, only marks, mark size=3.0pt, mark=o, mark options={solid, mycolor2}]
  table[row sep=crcr]{%
91.925581392	0.9513476782\\
};
\addlegendentry{\tiny{L2, $\gamma = 0.01$}}

\addplot [color=mycolor3, line width=1.5pt, only marks, mark size=3.0pt, mark=o, mark options={solid, mycolor3}]
  table[row sep=crcr]{%
45.687019862	0.9235636088\\
};
\addlegendentry{\tiny{L2, $\gamma = 0.05$}}

\addplot [color=mycolor4, line width=1.5pt, only marks, mark size=3.0pt, mark=o, mark options={solid, mycolor4}]
  table[row sep=crcr]{%
33.871171056	0.9064099658\\
};
\addlegendentry{\tiny{L2, $\gamma = 0.1$}}

\addplot [color=mycolor5, line width=1.5pt, only marks, mark size=3.0pt, mark=x, mark options={solid, mycolor5}]
  table[row sep=crcr]{%
0.9995026852	0.6539675348\\
};
\addlegendentry{\tiny{Lip1}}

\addplot [color=mycolor6, line width=1.5pt, only marks, mark size=3.0pt, mark=x, mark options={solid, mycolor6}]
  table[row sep=crcr]{%
1.9990605572	0.7306757266\\
};
\addlegendentry{\tiny{Lip2}}

\addplot [color=mycolor7, line width=1.5pt, only marks, mark size=3.0pt, mark=x, mark options={solid, mycolor7}]
  table[row sep=crcr]{%
4.9936367292	0.847610419\\
};
\addlegendentry{\tiny{Lip5}}

\addplot [color=mycolor1, line width=1.5pt, only marks, mark size=3.0pt, mark=x, mark options={solid, mycolor1}]
  table[row sep=crcr]{%
9.852389112	0.8998565496\\
};
\addlegendentry{\tiny{Lip10}}

\addplot [color=mycolor2, line width=1.5pt, only marks, mark size=3.0pt, mark=x, mark options={solid, mycolor2}]
  table[row sep=crcr]{%
18.737256618	0.9252246132\\
};
\addlegendentry{\tiny{Lip20}}

\addplot [color=mycolor3, line width=1.5pt, only marks, mark size=3.0pt, mark=x, mark options={solid, mycolor3}]
  table[row sep=crcr]{%
45.306912144	0.9471800682\\
};
\addlegendentry{\tiny{Lip50}}

\addplot [color=mycolor4, line width=1.5pt, only marks, mark size=3.0pt, mark=x, mark options={solid, mycolor4}]
  table[row sep=crcr]{%
77.970043432	0.9510003774\\
};
\addlegendentry{\tiny{Lip100}}

\addplot [color=mycolor5, dashed, forget plot]
  table[row sep=crcr]{%
1	0.6\\
1	1\\
};
\addplot [color=mycolor6, dashed, forget plot]
  table[row sep=crcr]{%
2	0.6\\
2	1\\
};
\addplot [color=mycolor7, dashed, forget plot]
  table[row sep=crcr]{%
5	0.6\\
5	1\\
};
\addplot [color=mycolor1, dashed, forget plot]
  table[row sep=crcr]{%
10	0.6\\
10	1\\
};
\addplot [color=mycolor2, dashed, forget plot]
  table[row sep=crcr]{%
20	0.6\\
20	1\\
};
\addplot [color=mycolor3, dashed, forget plot]
  table[row sep=crcr]{%
50	0.6\\
50	1\\
};
\addplot [color=mycolor4, dashed, forget plot]
  table[row sep=crcr]{%
100	0.6\\
100	1\\
};
\end{axis}
\end{tikzpicture}%
\caption{Accuracy robustness tradeoff. Test accuracy over Lipschitz constant from SDP \cite{pauli2023lipschitz}, averaged over 5 CNNs.}
\label{fig:robacc}
\end{figure}

\section{Conclusion}\label{sec:conclusion}
In this paper, we introduced a parameterization for Lipschitz-bounded 1D CNNs using Cayley transforms and controllability Gramians. Using our parameterization we can train Lipschitz-bounded 1D CNNs in an unconstrained training problem which we illustrated in the classification of ECG data from the MIT-BIH database. Future research includes the extension of our parameterization to 2D CNNs using a 2D systems approach as suggested in \cite{gramlich2022convolutional}.

\bibliographystyle{IEEEtran}
\bibliography{references}

\end{document}